\documentclass[11pt]{article}

\usepackage[utf8]{inputenc} 
\usepackage[T1]{fontenc}    
\usepackage{url}            
\usepackage{amsfonts,amsmath}       
\usepackage{nicefrac}       
\usepackage{microtype}      
\usepackage{float}

\usepackage{fullpage}
\usepackage{mathrsfs}
\usepackage[mathscr]{euscript}

\newtheorem{theorem}{\bf Theorem}

\def\QED{\rule[-1pt]{5pt}{5pt}\par\medskip}
\newenvironment{proof}{{\bf Proof: \ }}{ \hfill \QED}
\usepackage{tikz}
\usetikzlibrary{patterns}
\usetikzlibrary{arrows,shapes,trees,calc} 
\tikzstyle{block} = [rectangle, draw, fill=blue!20, text width=4em, text centered, rounded corners, minimum height=2em]
\tikzstyle{oplus} = [circle, draw, fill=red!20, text width=0.3em, text centered, rounded corners, minimum height=0em]

\let\ALP  \mathcal

\newcommand{\beq}[1]{\begin{eqnarray} #1 \end{eqnarray}}
\newcommand{\beqq}[1]{\begin{eqnarray*} #1 \end{eqnarray*}}
\renewcommand{\Re}{\mathbb{R}}

\newcommand{\ex}[1]{\mathbb{E}\left[#1\right]}

\DeclareMathOperator*{\argmax}{arg max}

\title{Adversarial Reinforcement Learning for Observer Design in Autonomous Systems under Cyber Attacks}

\author{
  Abhishek Gupta\\
  Electrical and Computer Engineering\\
  The Ohio State University\\
  Columbus, OH 43210\\
  \texttt{gupta.706@osu.edu} \\
   \and
  Zhaoyuan Yang \\
  Electrical and Computer Engineering\\
  The Ohio State University\\
  Columbus, OH 43210\\
  \texttt{yang.2507@buckeyemail.osu.edu} \\
}

\begin{document}

\maketitle

\begin{abstract}
Complex autonomous control systems are subjected to sensor failures, cyber-attacks, sensor noise, communication channel failures, etc. that introduce errors in the measurements. The corrupted information, if used for making decisions, can lead to degraded performance. We develop a framework for using adversarial deep reinforcement learning to design observer strategies that are robust to adversarial errors in information channels. We further show through simulation studies that the learned observation strategies perform remarkably well when the adversary's injected errors are bounded in some sense. We use neural network as function approximator in our studies with the understanding that any other suitable function approximating class can be used within our framework.
\end{abstract}

\section{Introduction}
In 2009, Air France Flight 447 crashed in Atlantic Ocean, killing everyone on board, due to inconsistency in airspeed readings from multiple sensors. The inconsistency in reading was attributed to ice crystal growth in pitot tubes, which corrupted the measurement of the airspeed. Based on corrupted measurements, the pilots took actions that forced the airplane into deep stall, after which the pilots were unable to regain control over the airplane and it crashed. It should be noted that the flight was on autopilot mode before the pilots were alerted to take control of the aircraft after the on board computer detected inconsistent measurements.

Such incidents in future autonomous systems cannot be avoided. One of the key issue in widespread deployment of future autonomous systems are that some of the sensors may fail to provide the correct information or the sensor measurement is tampered with by a strategic adversary\footnote{We freely assume in this paper that nature could be an adversarial agent.}. Such attacks, called deception attack, may induce the controller to take an action that is detrimental to the system performance. In order to keep the system safe from such attacks, one needs to design filtering scheme (also called observer design in control theory literature) to automatically identify bad data/measurements (which may be based on the system model), and reverse the effects of the attack (that is, compute the correct state of the system from the corrupted measurements). One of the key challenge in this situation is that the observer may not know {\it a priori} if the data/measurements are corrupted or not.

As a result, many authors have studied methods to detect a deception attack on the control system, and have devised dynamic filtering schemes to obtain the true state from corrupted measurements. Kalman filtering and particle methods are some of the techniques. If the measured information lie on a hyperplane, then principal component analysis can also be used to detect bad data \cite{xie2014}. However, these filtering schemes depend strongly on the model of the control system, that is, the system's accurate model is explicitly needed to design the filtering scheme. However, for very complex systems like humanoid robots or autonomous cars, it may be difficult to derive accurate system model from basic physical principles. Consequently, one may be interested in designing a model-free filtering scheme that detect corrupted measurements and automatically corrects the errors introduced by the adversary. We view our paper as a contribution to this important problem.

In our study, we assume that the state measurements have higher dimensions than the actual state of the system. Thus, there is inherent redundancy in the measurement space, allowing the observer to detect low dimensional adversarial attacks. Since the system follows rules of the physics, the measurements lie in a lower dimensional (potentially nonlinear) manifold in the measurement space. Thus, the goal of the observer is to learn a function that projects corrupted measurements (which could lie anywhere in the measurement space) to that lower dimensional manifold.
\begin{figure}[H]
	\centering
    \includegraphics[width=4in, trim = {180 60  200  80}, clip]{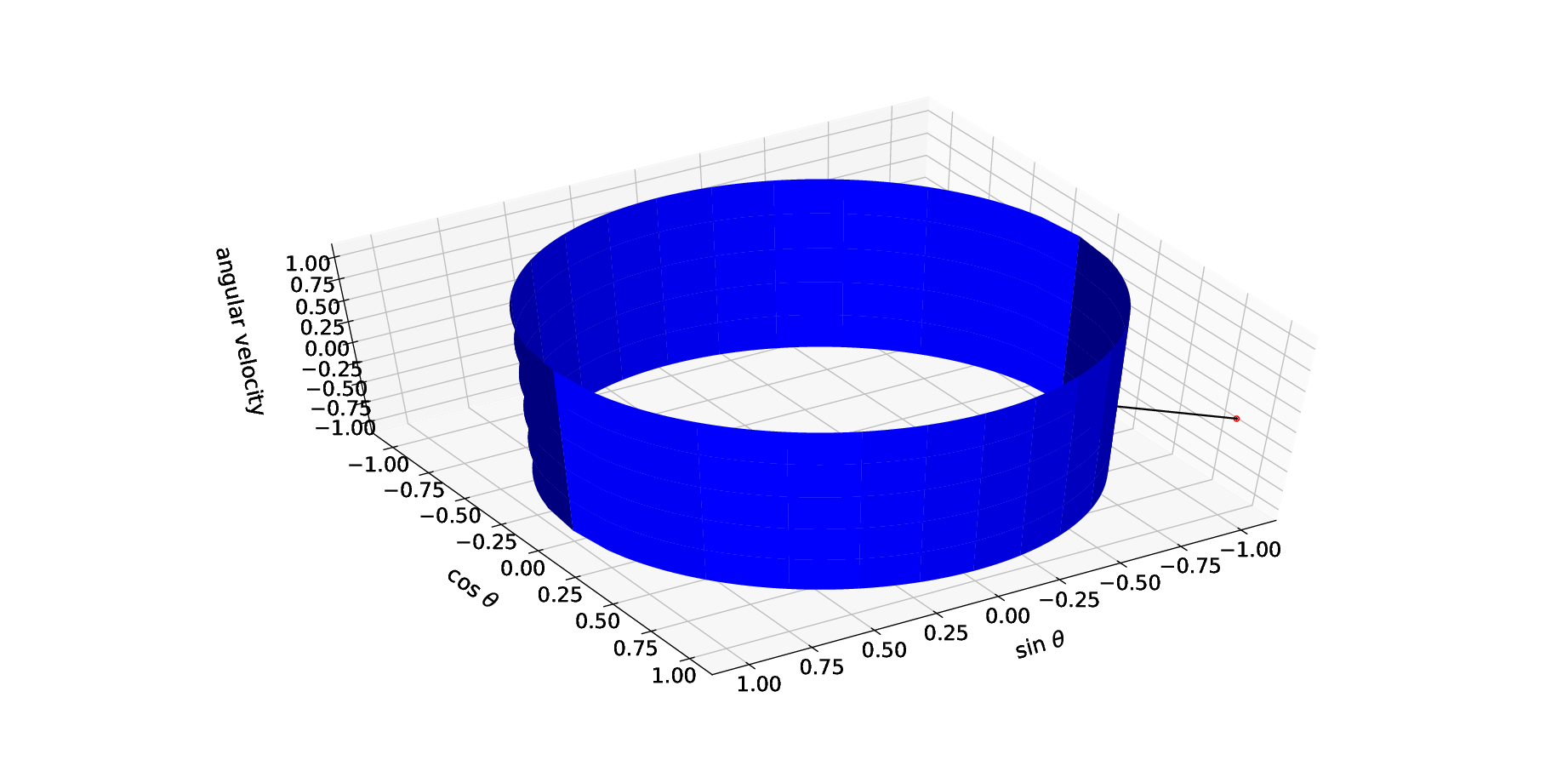}
\caption{\label{fig:manifold} The manifold on which measurements of an inverted pendulum lie. The point in space away from the manifold is the observed measurements corrupted due to adversarial noise.}
\end{figure}
As a concrete example, consider the classical inverted pendulum on a cart problem. Suppose that the inverted pendulum makes an angle $\theta$ with the vertical axis, and we have two measurements $y_1 = \cos(\theta)$ and $y_2 = \sin(\theta)$. Naturally, the measurement space is $[-1,1]^2$, but the actual measurements lie in the set $y_1^2+y_2^2 = 1$. Thus, if the adversary perturbs the measurements slightly, then it is easy to obtain an estimate of the state that is close to the true state. Indeed, if we do not know the system model and the measurement functions accurately, then we need an algorithmic approach to identify the lower dimensional subspace in which the true measurements of the states would lie. We use reinforcement learning, with neural network used as function approximator, to identify such a manifold in this paper. In Figure \ref{fig:manifold}, we show such a manifold for the classical single inverted pendulum in which we plot the manifold in which the position, speed, and angular velocity lie. Our work in intimately connected to several strands of work in the literature as explicated below.

\subsection{Deep Reinforcement Learning}
Advances in deep learning has allowed researchers to use neural networks as function approximators for reinforcement learning in systems with extremely large state and action spaces\cite{DDPG,PPO_ROBOSCHOOL,A3C,DQN,MCTSRLDL}. In these situations, neural network is used to store the value function, learned policy, and/or the Q function. If the action space is continuous, then the reinforcement learning algorithm generally suffers from low data efficiency and unstable performance during training. In benchmarking paper \cite{benchmark}, result shows that among the algorithms being implemented, Truncated Natural Policy Gradient (TNPG), Trust Region Policy Optimization (TRPO) \cite{TRPO}, and Deep Deterministic Policy Gradient (DDPG) \cite{DDPG} are effective methods for training deep neural network policies for continuous control problems. In our study, we are using neural network as observer, which maps continuous corrupted observations to continuous filtered observations.  

We use TRPO for implementing the observer design. Through solving constrained optimization with KL divergence as constraints, the algorithm tends to give monotonic improvement with only little changes in the hyperparameters \cite{TRPO}.

\subsection{Adversarial Example and Reinforcement Learning}
Neural network are vulnerable to adversarial attacks \cite{NN_adv_example,trafficSymbol}. Fast gradient sign method \cite{FGSM} can be used to generate adversarial image noises for convolutional neural network, and authors explain that cause of adversarial example is linear nature of neural network. In the context of reinforcement learning featuring attacks on neural network policies, authors in \cite{FGSMATTACK} use fast gradient sign method \cite{FGSM} to generate adversarial noise to image input, and the corrupted image is used for the control task. It turns out regardless of which environment the policy is trained for or how it is trained, it is possible to significantly decrease the performance of policy \cite{FGSMATTACK}. Besides adding adversarial noise on inputs, adversarial noise can also be added by applying adversarial force on physical systems \cite{ADVRL} or adding  noise to corrupt actions made by controller \cite{ADVACTION}. It turns out adding adversarial noise while training will make reinforcement learning more robust to environmental noise. Imitation learning can also benefit from adding adversarial noise while supervisors making demonstrations \cite{DART}.

\subsection{Neural Network as State Estimator/Observer}
In continuous control problems, neural network is typically used to store policy or value function, but they can also be used as an observer. Neural network observers can be used to estimate unavailable system states, as has been investigated in \cite{NN_observer} and \cite{NN_observer2}. Compared with Kalman filters for sensor failure detection, identification, and accommodation, neural networks' performance for sensor failure detection, identification, and accommodation are better when applied to any dynamic system where system model and filter's model mismatch \cite{NN_KF}. Reinforcement learning can also be used as state estimator to estimate hidden variables and parameters of nonlinear dynamic systems\cite{RLSE}.

\subsection {Deception Attack of Autonomous System} 
Deception attacks are cyber attacks on autonomous systems in which a strategic adversary injects noise in the communication channel between the sensor and the command center. Such attacks on autonomous systems have been investigated in \cite{amin2010, gupta2011, amin2013, teixeira2012}. In complex cyberphysical systems, the nominal plant model and observation model may not capture all the nonlinearities in the measurements. A natural question to ask would be: Can we use the measurement data of the plant to identify the manifold in which the measurements lie?

Motivated by disparate strands of research on deep reinforcement learning and cyber attacks on autonomous systems, we investigate in this paper if deep reinforcement learning can be used to reliably design an observer for autonomous systems in a data-driven manner that can filter the effects of adversarial attacks. We fix the control policy, and use system's performance to train the observer. The adversary is simultaneously trained (as in fictitious play method used in game theory) to minimize the system's performance. In our simulation, we use neural networks as function approximators for both observer map and adversary's policy. However, some prior domain knowledge can be used for coming up with better/reliable function approximators depending on the application.

\section{Problem Formulation}
We consider a Markov decision problem in which the state, action, and actuation noise at time $t$ is denoted by $x_t$, $u_t$, and $w_t$, respectively. We use $\ALP X$ to denote the state space, $\ALP U$ to denote the action space, and $\ALP W$ to denote the actuation noise space. The state evolution equation is given by
\beqq{x_{t+1} = f(x_t,u_t,w_t), t= 1,2,\ldots.}
We assume that $x_0$ is the initial state and the noise distribution is known. Further, we assume that the optimal control policy $\gamma^*:\ALP X\rightarrow\ALP U$ is known and is Lipschitz continuous with Lipschitz constant $L_{\gamma^*}$.  

The state is measured through sensors and sent to the command center via a communication channel. The measurement and communication process opens up the information to sensor noise, sensor failures, packet errors, and cyber-attacks, which corrupt the information being sent to the command center. If the command center uses the corrupted information to make decisions, then the reward to the command center reduces. Thus, there is a need to robustly filter the errors in the information introduced due to exogenous sources to make optimal decisions that maximizes long term discounted reward.

To model this situation, we consider an adversary that takes as input the true state of the system and outputs a corrupted version of the state. Let $v_t$ and $z_t$ be uniformly distributed random variables in the set $[0,1]$ and are independent of each other. Let $g:\ALP X\times[0,1]\rightarrow\ALP X$ be the adversary's (behavioral) strategy and $\ALP G$ be the set of all possible adversary's (behavioral) strategies. The command center takes as input the corrupted state $y_t = g(x_t,v_t)$ and generates an estimate $\hat x_t = h(y_t,z_t)$ using estimation rule $h:\ALP X\times [0,1]\rightarrow\ALP X$. We use $\ALP H$ to denote the set of all possible estimation rules. We note here that the observer can potentially use randomized estimation rules as well (in many games, it is beneficial to act according to randomized rules to introduce robustness to other player's biases)

\begin{figure}[h!]
\begin{center}
\tikzstyle{block} = [draw, fill=blue!20, rectangle, minimum height=3em, minimum width=6em]
\begin{tikzpicture}[node distance=2cm,>=latex']
    \node [block] (controller) at (0,0){Observer};    
    \node [block] (environment) at (4.5,0) {Controller};
    \node [block] (measurements) at (3.5,-1.5) {Environment};
    \node [block] (observer) at (-1.5,-1.5) {Adversary};
    \draw [->] (controller) -- (environment) node[pos=0.5, below] {$\hat x_t = h(y_t,z_t)$};
    \draw [->]  (environment) -- (6,0) --(6,-1.5)--(measurements) node[pos=0.5, below] {$u_t$};
    \draw [->]  (measurements)--(observer) node[pos=0.5, below] {$x_t$};
    \draw [->]  (observer)--(-3.5,-1.5)--(-3.5,0)--(controller) node[pos=0.5, below] {$y_t = g(x_t,v_t)$};
\end{tikzpicture}
\caption{Observer to recover true state. Objective of observer is to estimate true state returned from environment such that controller can maintain its good performance. }
\end{center}
\end{figure}
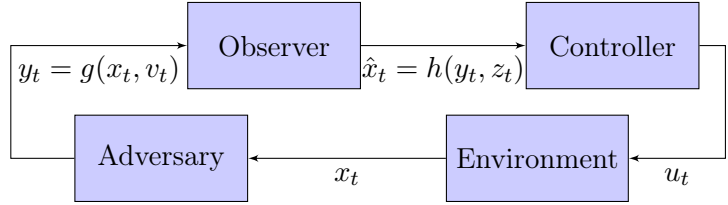

The reward of the command center is given by
\beqq{J(g,h):= J_{x_0,\gamma^*}(g,h)  = \ex{\sum_{t=0}^\infty \alpha^t r(x_t,u_t)\bigg|u_t = \gamma^*(h(g(x_t,v_t),z_t))}.}

The goal of the command center is to maximize the total discounted reward by picking an appropriate estimation map $h$, whereas the goal of the adversary is to minimize the total discounted reward. Thus, we arrive at a zero-sum game, where we are interested in the saddle-point equilibrium $(g^*,h^*)$ such that
\beqq{J(g,h^*)\leq J(g^*,h^*)\leq J(g^*,h).}
In this paper, we refer to the command center and the adversary to be the players of this zero-sum game.
%
%
%
%

\section{Fictitious Play and A Function Approximation Approach}
Fictitious play is a iterative best response method in game theory in which at each time, each player forms an empirical belief on the actions taken by the other player, and plays the best response to that belief. It has been shown to converge in zero-sum game with finite strategy spaces in \cite{brown1951}.

However, in the problem at hand, the strategy spaces of the adversary and the command center, $\ALP G$ and $\ALP H$, are uncountable spaces, even if the state and action spaces of the command center are finite. Moreover, if the state and action spaces are uncountable (as in inverted pendulum), then the strategy spaces are infinite dimensional. Thus, we need to resort to function approximation, wherein we pick a suitable choice of parametrized function spaces to restrict the strategy spaces of the players. We let $\Theta\subset\Re^n$ denote the parameter space of the adversary's parametrized strategy set $\ALP G(\Theta):=\{g_\theta:\ALP X\times[0,1]\rightarrow\ALP X\}$. Similarly, we let $\Xi\subset\Re^m$ denote the parameter space of the command center's parametrized estimator set $\ALP H(\Xi):=\{h_\xi:\ALP X\times[0,1]\rightarrow\ALP X\}$. 

\subsection{Existence of Approximate Nash Equilibrium}
Since we are using function approximation using neural network in this paper, we need to understand if fictitious play can lead to an approximate saddle-point equilibrium in the game formulated above. We have the following result on this issue. 
\begin{theorem}\label{thm:main}
Fix $\epsilon>0$. Suppose that 
\begin{enumerate}
\item For every $h\in\ALP H$, there exists $h_\xi\in\ALP H(\Xi)$ such that
\beqq{|J(g,h)-J(g,h_\xi)| \leq \epsilon\quad \text{ for all } g\in\ALP G.}
\item For every $g\in\ALP G$, there exists $g_\theta\in\ALP G(\Theta)$ such that
\beqq{|J(g_\theta,h)-J(g,h)| \leq \epsilon \quad \text{ for all } h\in\ALP H.}
\end{enumerate}
Then, for every saddle-point equilibrium $(g^*,h^*)$ of the game, there exists a pair of functions $(g_{\theta^*},h_{\xi^*})$ that is $3\epsilon$-Nash equilibrium of the game. 
\end{theorem}
\begin{proof}
First, construct a pair of functions $(g_{\theta^*},h_{\xi^*})$ that is at most $\epsilon$ away from $(g^*,h^*)$. We need to prove that for any $(g,h)\in\ALP G\times\ALP H$, we have
\beq{J(g_{\theta^*},h_{\xi^*}) \leq J(g_{\theta^*},h)+3\epsilon,\label{eqn:gtheta}\\
J(g_{\theta^*},h_{\xi^*}) \geq J(g,h_{\xi^*})-3\epsilon.\label{eqn:hxi}}
The following two set of inequalities follows from the assumptions:
\beqq{& J(g^*,h^*) \leq J(g^*,h) \leq J(g_{\theta^*},h)+\epsilon,\\
& J(g^*,h^*) \geq J(g_{\theta^*},h^*)- \epsilon \geq J(g_{\theta^*},h_{\xi^*})-2\epsilon.}
Collecting the inequalities above, we get \eqref{eqn:gtheta}. The inequalities in \eqref{eqn:hxi} can also be obtained in a similar fashion.
\end{proof}
Thus, by picking an appropriate function approximator class for the adversary and the observer, we can reach as close to the saddle-point equilibrium as possible.

\subsection{Sequential and Simultaneous Policy Update Rules in Games}
We are now in a position to define fictitious play for our observer design. Let us define $\tilde J(\theta,\xi) := J(g_\theta,h_\xi)$. The observer map and adversary's best response can be trained through sequential update scheme or simultaneous update scheme. First, we can pick $\theta_0$ and $\xi_0$ arbitrarily. The traditional simultaneous and sequential update schemes in zero-sum games are defined as follows:

{\bf Sequential Updates:}
\beqq{ \theta_{k+1} = \begin{cases}\theta_k +\eta \nabla_\theta \tilde J(\theta_k,\xi_k) & \text{ $k$ is odd}\\ \theta_k & \text{ $k$ is even}\end{cases},\quad
 & \xi_{k+1} = \begin{cases}\xi_k -\eta \nabla_\xi \tilde J(\theta_k,\xi_k) & \text{ $k$ is even}\\ \xi_k & \text{ $k$ is odd}\end{cases}.} 

{\bf Simultaneous Updates:}  
\beqq{& \theta_{k+1} = \theta_k +\eta\nabla_\theta \tilde J(\theta_k,\xi_k), \quad & \xi_{k+1} = \xi_k -\eta \nabla_\xi \tilde J(\theta_k,\xi_k).}

Typically, if the function $\tilde J$ is strictly convex in $\xi$ and strictly concave in $\theta$, then the update schemes are known to converge \cite{rosen1965}. In this paper, since we use neural networks as function approximators, the map $\tilde J$ does not satisfy the strict concave/convex conditions. As a result, the fictitious play may not converge.

In the problem at hand, computing $\tilde J$ for every value of $\theta\in\Theta$ and $\xi\in\Xi$ is not possible. Thus, we use TRPO with generalized advantage estimate to compute the next iterate \cite{schulman2015}. In this method, we compute $\theta_{k+1}$ by solving 
\beqq{\theta_{k+1} = \underset{\theta\in\Theta}{\argmax}\;  \frac{1}{N}\sum_{n=1}^{N}\Bigg[\frac{{g_{\theta}}(y_n|x_n)}{g_{\theta_k}(y_n|x_n)}A^{g_{\theta_k}}(x_n,y_n)\Bigg]\\
\textrm{ subject to } \frac{1}{N}\sum_{n=1}^{N}\Big[D_{KL}\big(g_{\theta_k}(\cdot|x_n)\big\|{g_{\theta}}(\cdot|x_n)\big)\Big] \leq\delta,  }


where $x_1$ is initialized randomly, and the generalized advantage estimate $A^{g_{\theta_k}}$ is given by
\beqq{A^{g_{\theta_k}}(x_n,y_n) =  (1-\lambda)\sum_{l=0}^\infty (\lambda\alpha)^l \Big(r(x_{n+l},u_{n+l})+\alpha V_{\zeta_k}(x_{n+l+1})-V_{\zeta_k}(x_{n+l})\Big),\\
\text{with } V_n=\sum_{t=0}^{\infty} \alpha^t r(x_{n+t},u_{n+t}) \text{ and } V_{\zeta_k} = \underset{\zeta}\min\Big(V_n(x_n) - V_{\zeta}(x_n)\Big)^2.}


In the computation of the advantage estimate for computing $\theta_{k+1}$, we fix $\xi_k$. A similar approach is adopted for computing the next iterate $\xi_{k+1}$ from fixed $\xi_k$ and $\theta_k$. 

%

\section{Simulation Results on Inverted Pendulum and Key Observations} 
We now demonstrate performance of the reinforcement learning for learning adversary's policy and observer's map in three different settings: (a) performance of adversary without an observer; (b) performance of adversary and observer in sequential update setting; (c) performance of adversary and observer in simultaneous update setting. In following sections, Trust Region Policy Optimization from OpenAI Baselines \cite{openaibaselines} is used to optimize the players' policies. 
\begin{figure}[bth]
\centering
\scalebox{0.6}{
\begin{tikzpicture}[thick,>=latex,->]
\begin{scope}
\clip(-7,3.7) rectangle (7,-2.85);
\draw[dashed] (0,0)  circle (3.24cm);
\filldraw[white] (-4.3,0) rectangle (4.3,-4.3);
\draw[double distance=1.6mm] (0,0) -- (1,3) node[midway,xshift=4mm,yshift=2mm]{};
\draw[fill=white] (1.2,-1.0) -- (0.5,0) arc(0:180:0.5) -- (-1.2,-1.0) -- cycle;

\draw[draw=black,fill=white] (0, 0) circle circle (.3cm);
\draw[draw=black,fill=white] (1,3) circle circle (.3cm);

\draw[pattern=north east lines] (-1.4,-1.5) rectangle (1.4,-1);
\draw[draw=black,fill=white] (0.8,-1.7) circle circle (.4cm);
\draw[draw=black,fill=white] (-0.8,-1.7) circle circle (.4cm);
\draw[draw=black,fill=white] (-0.8,-1.7) circle circle (.05cm);
\draw[draw=black,fill=white] (0.8,-1.7) circle circle (.05cm);
\draw[dashed,-] (0,-2.7) -- (0,3.5);
\draw[dashed,-] (-7,-2.1) -- (7,-2.1);
\node at (.2,1.6) {\Large $\theta$};   
\draw[->] (1.6,-1.5) -- (2.6,-1.5) node[below]{\Large $u_t$};
\draw[->] (-2.6,-2.5) -- (0,-2.5) node[pos = 0.6, below]{\Large $x$};
\draw[->] (-2.6,-2.5) -- (-2.6,3.7);
\end{scope}
\end{tikzpicture}}
 \caption{\label{fig:cartpole} An inverted pendulum, where $x$ is the displacement from the origin, $\theta$ is angle from the vertical, and $u_t$ is the force on the cart.}
\end{figure}
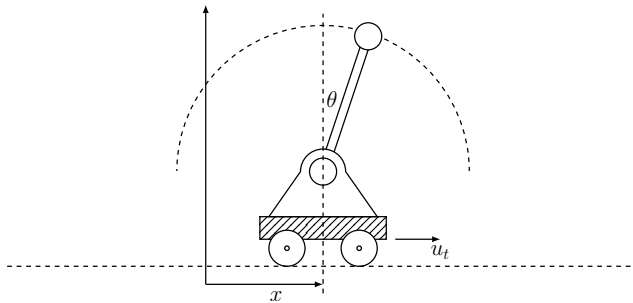

\subsection{Environment}
We use inverted pendulum from Roboschool \cite{PPO_ROBOSCHOOL} as simulation environment. Environment is shown in Figure \ref{fig:cartpole}. The inverted pendulum's state space has three dimensions -- location $x$, velocity $v_x$, and the angle from the vertical $\theta$, and the rotational speed $\dot\theta$. The measurement space has 5 dimensions -- location $x$, speed $v_x$, $\cos \theta$, $\sin \theta$ and angular velocity $\dot{\theta}$. Reward for the system is $r(\theta) = 0.08- |\sin \theta|$. Since $\theta$ is the only payoff relevant state, we assume that the adversary attacks the measurement $\sin \theta$. The control action on the pendulum is the force $F$. Roboschool provides pre-trained model for different environments, and we use pre-trained model of inverted pendulum from OpenAI to analyze performance of adversary agent. 

In what follows, each episode ends at either 1000 time steps or when $|\theta|>0.2$, whichever is earlier. Thus, when the adversary learns the policy that can push $|\theta|>0.2$ in 200 time steps, then the episode ends at 200. 


\subsection{Performance of Controller with Adversary and without Observer}
Among all measurements, $\sin \theta$ is the most crucial measurement, since it relays information about the angle of the pendulum. Thus, we let the adversarial agent to corrupt $\sin \theta$ through an adversarial noise. Reward for adversarial agent is defined as  $|\sin \theta| -0.08$. We assume that the adversarial noise is bounded as follows. We compute the standard deviation $\sigma$ of $\sin \theta$ in the noiseless environment in 1000 trials, and we bound the adversarial noise magnitude to be between $[-2\sigma,2\sigma]$. We empirically found $\sigma = 0.00789366$.

Without observer, adversary learns a policy that reduces the performance of controller through generating adversarial noise for $\sin \theta$. Performance of adversary in each episode is shown in Figure \ref{fig:no_observer}. Each episode has at most 1000 time steps. Also note that since we have 1000 time steps in each episode, if $\theta$ is small all the time (which implies the inverted pendulum is perfectly balanced by the controller), then the total accumulated reward by the controller at the end of the episode is $0.08*1000 = 80$. Naturally, the minimum possible reward the adversary can earn in each episode is also -80. The maximum reward adversary can earn is 0, which it can secure if it induces the pendulum to fall in the first time step itself (if pendulum falls in first step, then the reward for adversary will be zero since episode has 0 time steps). 
%

\begin{figure}[H]
	\centering
    \includegraphics[width=5in]{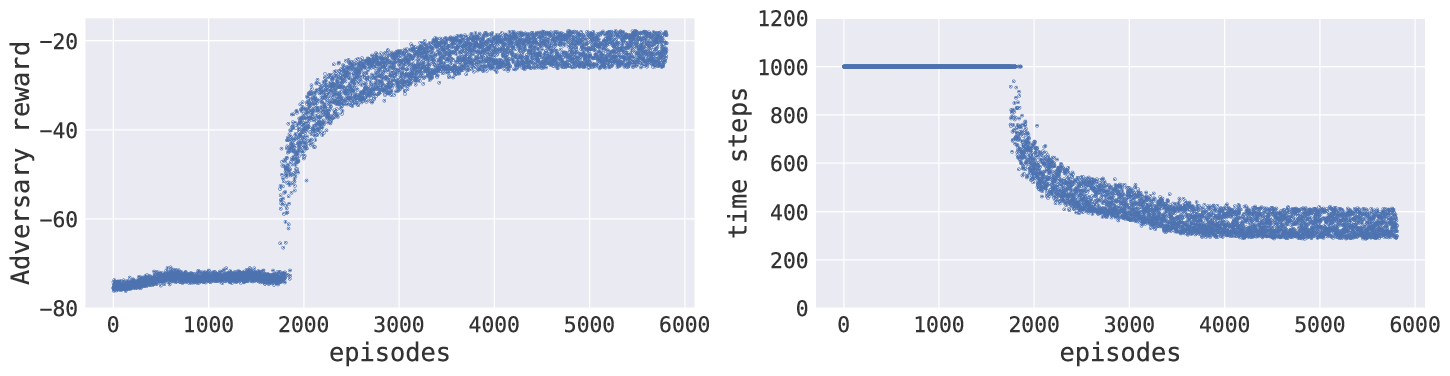}
\caption{\label{fig:no_observer} The left figure depicts the evolution of reward of the adversary while learning, and the right figure depicts the number of time steps in each episode.}
\end{figure}
We observe in the Figure \ref{fig:no_observer} that in the absence of the observer (or the observer is fixed to be an identity map), the adversary was able to learn the worst case adversarial noise that leads to the worst performance. In the initial learning phase (episodes 1-1900), the adversary's noise did not affect the final control action very much and its accrued reward in each episode was roughly equal to $-1000\times 0.08$. After about 1900 episodes, the adversary's performance improves dramatically as it learns how to sway the observation in a manner that reduces the controller's performance.

\subsection{Fictitious Play with Adversary and Observer}
We count 15000 time steps as one iteration in the fictitious play and perform both sequential updates and simultaneous updates for 5 trials. Bounds for adversarial noise are set to $2\sigma$, $4\sigma$, and $8\sigma$ in the following analysis. The results are as follows:

\textbf{Sequential Update}: We run 5 trials for each bound on adversarial noise magnitude. Since adversary and observer learn alternatively during the simulation, each player updates its neural network parameters for 15000 steps against the fixed policy of the other player, and then the other player starts learning. Results of all 5 trials are similar, so we only show one of the results as a representative learning behavior for varying bounds on adversarial noise in Figure \ref{fig:sample_seq1}.

\begin{figure}[H]
\centering
\includegraphics[width=4in]{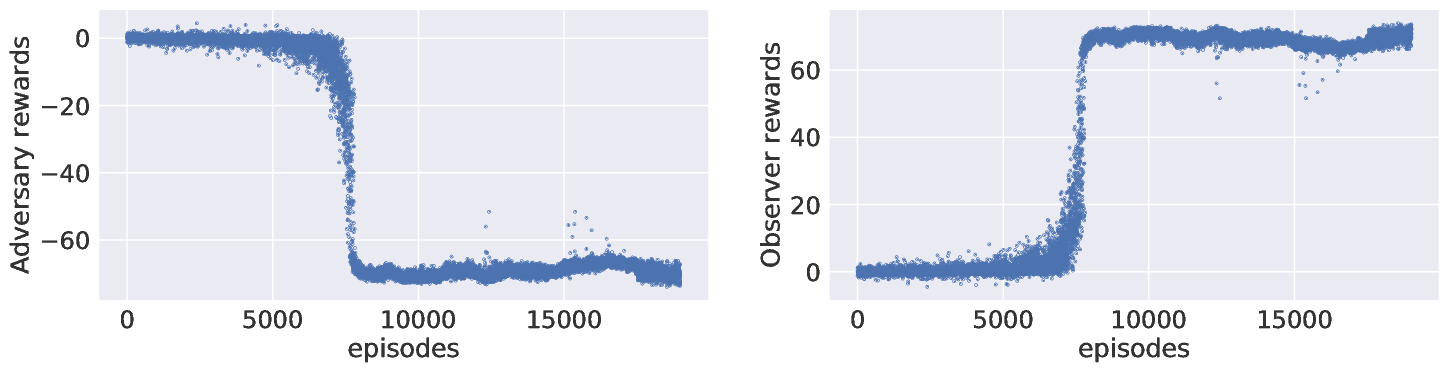}\\
\includegraphics[width=4in]{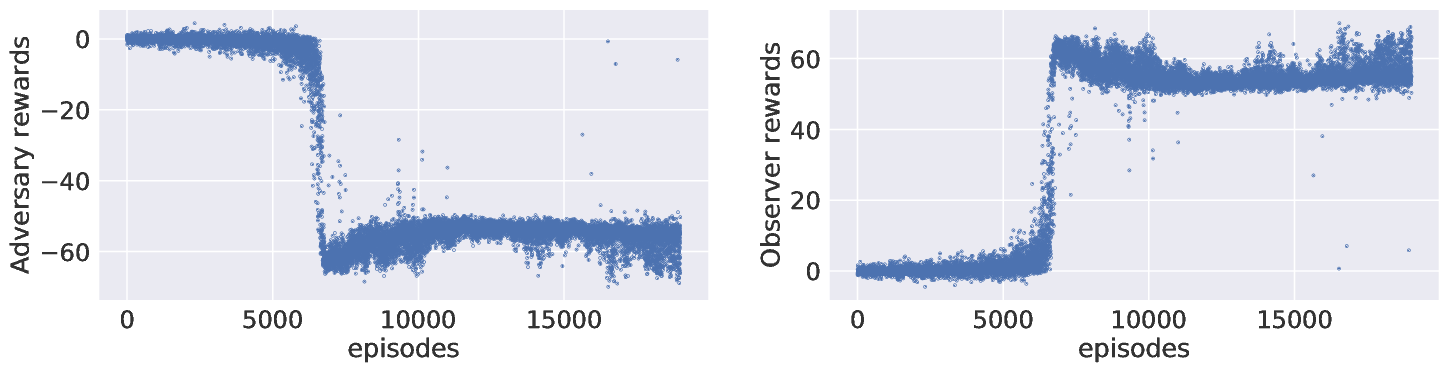}
\includegraphics[width=4in]{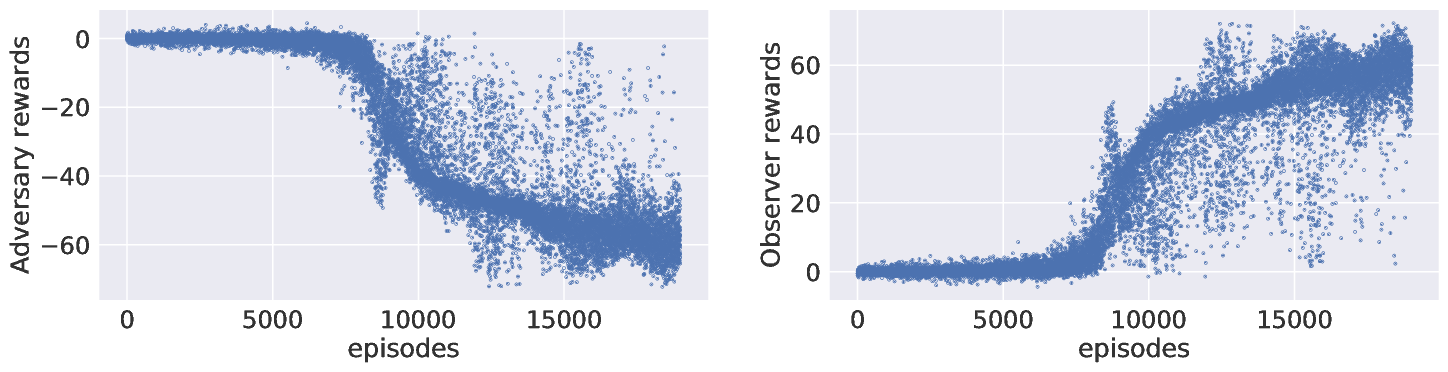}
\caption{\label{fig:sample_seq1}
The plot of cumulative rewards vs. episodes for the adversary and the observer. A sample trial of sequential update with adversarial noise in $[-2\sigma, 2\sigma]$ in the top plot, $[-4\sigma, 4\sigma]$ in the middle plot, and $[-8\sigma, 8\sigma]$ in the bottom plot.}
\end{figure}


\textbf{Simultaneous Update}: Similar to sequential update, we ran 5 trials, and since the results were similar, we present the representative plots below in Figure \ref{fig:sample_sim1} with different bounds on adversarial noise. 

\begin{figure}[H]
\centering
\includegraphics[width=4in]{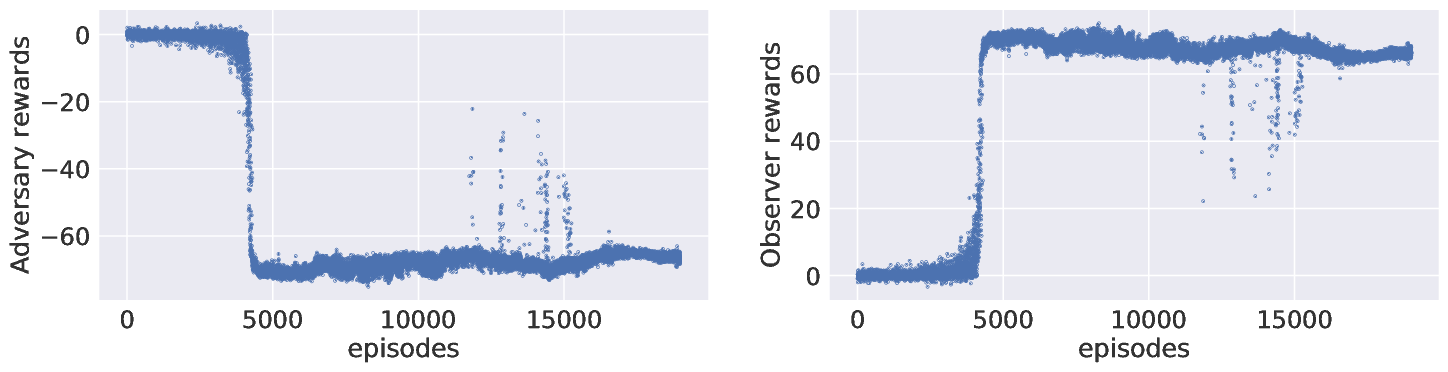}\\
\includegraphics[width=4in]{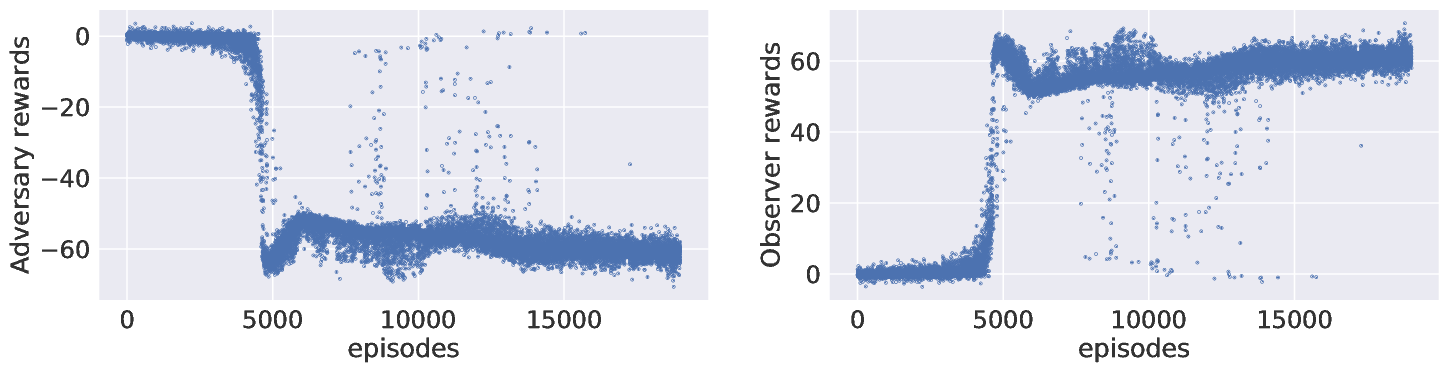}\\
\includegraphics[width=4in]{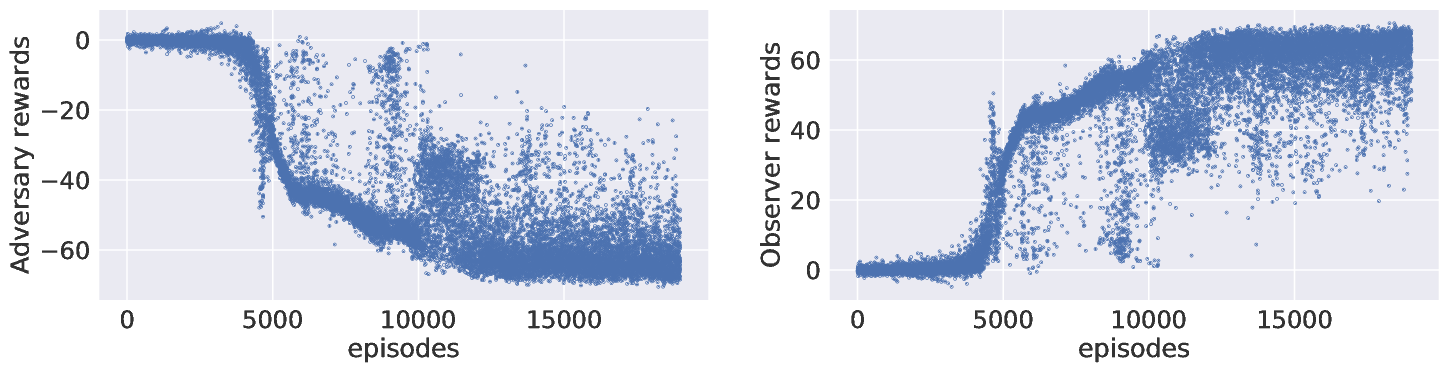}
\caption{\label{fig:sample_sim1}
The plot of cumulative rewards vs. episodes for the adversary and the observer. A sample trial of simultaneous update with adversarial noise in $[-2\sigma, 2\sigma]$ in the top plot, $[-4\sigma, 4\sigma]$ in the middle plot, and $[-8\sigma, 8\sigma]$ in the bottom plot.}
\end{figure}

An interesting observation in the plots above is that for small values of adversarial noise, the observer is in most cases able to reverse the corruption introduced by adversary. Consequently, the payoff to the system in each episode is close to the maximum attainable payoff (about $80$) after the observer is trained. We note here that the observer need not compute the uncorrupted state exactly; as far as the observer is able to compute the estimate that, when fed into the controller, takes the same action as in the uncorrupted state case, the final payoffs are the same. The abrupt jumps or drops in the payoffs of the players in certain episodes during the learning phase are due to the other player learning the best response strategy very well in the corresponding iteration of the fictitious play. If the adversary is strong enough to corrupt the measurements by large quantity, then the data-driven observer design using our method is not successful. We intend to address this very interesting problem in our future research.

\section{Conclusion}
Autonomous systems must be resilient to cyber attacks and sensor failures. In modern systems, this is typically attained through introducing redundancy in measurements and designing an observer, which takes into account the physical model of the plant and determines an estimate of the state based on the measurements. Kalman filtering scheme and its variants presents a method to do it if the model of the system is known. In our work, we studied the problem of observer design through measurement data and in the presence of an adversary for a simple autonomous system--an inverted pendulum. We believe that the methods developed in this paper can be extended to more complex systems, which we will investigate in our future work. 

%

\bibliographystyle{plain}
\bibliography{bibfile}

\end{document}